\documentclass{article}
\usepackage{spconf,amsmath,graphicx}
\usepackage{amssymb}
\usepackage{soul}
\usepackage{xcolor}
\usepackage{mathtools}
\usepackage{enumitem}
\usepackage{amsthm}

\newtheorem{proposition}{Proposition}

\usepackage{subfig}

\usepackage[font=small,skip=2.pt]{caption}

\setlist{nosep, leftmargin=14pt}

\usepackage{mwe} 
\usepackage[colorlinks=true,]{hyperref}

\DeclareRobustCommand{\hllime}[1]{{\sethlcolor{lime}\hl{#1}}}

\AtBeginDocument{
  \hypersetup{
    urlcolor=red,
    citecolor=green,
    linkcolor=blue,
  }%
}

\title{Normal Reconstruction from Specularity in the Endoscopic Setting}
\name{Karim Makki, Adrien Bartoli\thanks{This work was funded by the FET-Open grant 863146 Endomapper.}}
\address{ EnCoV-Institut Pascal, CNRS/Université Clermont-Auvergne, Clermont-Ferrand 63000, France}
\begin{document}
%
\maketitle
\begin{abstract}
We show that for a plane imaged by an endoscope the specular isophotes are concentric circles on the scene plane, which appear as nested ellipses in the image. 
We show that these ellipses can be detected and used to estimate the plane's normal direction, forming a normal reconstruction method, which we validate on simulated data.
In practice, the anatomical surfaces visible in endoscopic images are locally planar.
We use our method to show that the surface normal can thus be reconstructed for each of the numerous  specularities typically visible on moist tissues. 
We show results on laparoscopic and colonoscopic images.
\end{abstract}
\begin{keywords}
Specular, reconstruction, endoscopy.
\end{keywords}
\section{Introduction}
\label{sec:intro}

Specularities typically appear as moving white spots on the tissues in endoscopy; they are usually considered as nuisance~\cite{daher2022temporal}. 
In contrast, we propose to use them as 3D reconstruction cues arising from a single image.
They may then be used as valuable assets to integrate in existing multi-image endoscopic 3D reconstruction systems~\cite{bell2013image,ozyoruk2021endoslam, sengupta2021colonoscopic}. Although shape-from-specularity has been researched in computer vision,
our proposal of using the specular isophotes for 3D reconstruction is new. 
Existing work is generally based on tracking specularities or reflections over multiple images. The analysis in~\cite{zisserman1989information} shows that local surface convexity or concavity can be determined from the image motion of specularities with a known camera trajectory. 
The method in~\cite{chen2006mesostructure} reconstructs the shape mesostructure from a dense specularity field from hundreds of images. 
The  method in~\cite{roth2006specular} uses the specular flow for estimating scene structure from 2D image motion. The method in~\cite{sankaranarayanan2010specular} uses a quadric surface model, an orthographic camera, and requires the object to rotate around the optical axis to reconstruct the shape of a mirror surface from sparse correspondences. 
The method in~\cite{liu2011pose} reconstructs the pose of a moving plane from three images of its reflection on a specular surface.
These methods rely on assumptions which cannot be fulfilled in interventional medical images. 

In contrast, we provide theoretical and practical insights into the 3D reconstruction of the observed tissue's normal from an isolated specularity.
We show that once a specularity is detected, the shape normal at the specularity's Brightest Point (BP) can be estimated in closed-form. 
We derive these results under the so-called \hllime{endoscopic setup, which is that the light is a point-light source collocated with the camera centre~(see} \cite[Section 2.3]{okatani1997shape} \hllime{and}~\cite[Section 2]{yeung1999global} \hllime{for details)}. 
Our main assumption is then that the observed surface is locally flat and that the reflectance is modelled by Phong's specular model.
We study the isophotes, which are the curves of equal incoming intensities on the camera's retinal plane.
We show that under the above assumptions the isophotes are concentric circles on the scene plane and ellipses in the image. 
From this key result, we use well-known 3D vision results on conics, of which the ellipse is a special case.
Specifically, we show that using a single ellipse isophote detected in an image, the scene plane's normal can be reconstructed up to a two-fold ambiguity.
This leads to a reliable method of local normal reconstruction for mildly curved surface patches.
We show results on synthetic and real images of the liver and the colon.



\section{Theory and methods}
\label{sec:methods}
\subsection{General Specular Isocurve Model}
\label{sssec:quartic model} 
A model of specular isophotes on planes was introduced~\cite{bartoli2019highlight} with Phong's reflection~\cite{phong1975illumination}, where it was shown that the specular isophotes are subset of quartic isocurves on the scene plane.
We briefly recall and adapt the derivation of this model and specialise it to the endoscopic setup.
The scene plane is chosen as the $XY$-plane with 2D points $p^\top=(x,y)$ corresponding to 3D points $P^\top=(x,y,0)$. 
Defining $V, L \in \mathbb{R}^3$ as the viewpoint and light source position, the cosine formula~\cite{phong1975illumination} gives specular reflection as: 
\begin{equation}
\hat{I}_s(p) \overset{\text{def}}{=} c \max (0,\cos(\beta(p)))^n ,
\label{Eq1}
\end{equation}
where $n>0$ characterises the surface roughness, $c>0$ is an unknown constant factor modelling the surface albedo, light intensity and linear camera response coefficient, and $\beta(p)$ is the angle between the viewing direction $\mu(V-P)$ and the direction of perfect reflection $-\mu(R-P)$ such that $\mu (U) \overset{\text{def}}{=} \frac{U}{\|U\|}$ for $U \in \mathbb{R}^3$ and $R^\top = (L_X, L_Y, -L_Z)$. 
By substituting these components in equation~(\ref{Eq1}), simplifying and introducing $\tau = \sqrt[n]{\frac{t}{c}}$, it was shown~\cite{bartoli2019highlight} that, on an infinite scene plane, the specular isophote $\hat I_s(p)=t$, $0<t\leq c$ is a subset of the quartic isocurve (that is, given by nullifying a bivariate polynomial of degree four) $Q(p)=0$, called the {\em general specular isocurve}, with:
\begin{equation}
Q(p) = ((R-P)^\top(V-P))^2 - \tau^2 \|R-P\|^2 \|V-P\|^2.
\label{Eq_quartic}
\end{equation}
We point out the distinction that exists between an {\em isophote} $\hat I_s(p)=t$, which is a curve on the scene plane of constant intensity as observed in the image, and its model as a quartic {\em isocurve} $Q(p)=0$.
The quartic isocurve is a pair of curves; parts of which form the isophote~\cite{bartoli2019highlight}, in a non-trivial and non-systematic way.
We show that simplifications exist in the endoscopic setup.

\subsection{Endoscopic Specular Isophote Model}
\label{sssec:endoscopic model} 
In the endoscopic setup, the light source and the camera are co-located and we thus set $L = V$.
We substitute $\kappa = 1 - \tau^2$ and $R^\top= (V_X,V_Y,-V_Z)$ in equation~\eqref{Eq_quartic}, which after simplification yields the {\em endoscopic specular isocurve} $E(p)=0$, with:
\begin{equation}
\begin{aligned}
E(p) = \kappa ((x-V_X)^2+(y-V_Y)^2)^2 + \\2V_Z^2(\kappa-2) 
((x-V_X)^2+(y-V_Y)^2) + \kappa V_Z^4.
\label{Eq6}
\end{aligned}
\end{equation}
This is still a quartic equation, albeit of a simpler form than the general  quartic isocurve~(\ref{Eq_quartic}).
Our main results are next given as two propositions, which characterise the geometric curve represented by this quartic.

\begin{proposition}
The endoscopic specular isocurve $E(p)=0$ 
represents a pair of concentric circles with centre $(V_X,V_Y)^\top$ and radii $r_\pm = |V_Z| \sqrt{ \frac{2-\kappa}{\kappa} \pm \sqrt{(\frac{2-\kappa}{\kappa})^2 -1}}$.
\label{proposition1}
\end{proposition}

\begin{proof}
The quartic~\eqref{Eq6} factorises into the product of two quadratics as:
\begin{equation}
E(p) = q_+(p)\,q_-(p), 
\label{Eq7}
\end{equation}
where $q_\pm(p)= (x-V_X)^2+(y-V_Y)^2 -r_\pm^2$ and $q_+,q_-$ represent the sought circles.
\end{proof}

\begin{proposition}
The endoscopic specular isophote is the circle of smallest radius $r_-$ from the endoscopic specular isocurve.
\label{proposition2}
\end{proposition}
\begin{proof}
For $0<t<c$ we have $0<\kappa<1$ and $r_+ > |V_Z| > r_- > 0$.
It was shown~\cite[lemma 27]{bartoli2019highlight} that, in the general case, the isophote lies in a so-called `outer ring', a circle of radius $|V_Z|$ and with centre $(V_X,V_Y)^\top$.
Therefore, in the endoscopic case, the isophote must be the smallest of the two circles of the isocurve.
\end{proof}

Figure~\ref{fig:fig1}(a) illustrates the isocurves obtained for simulated data, as described in section~\ref{sec:results}.
We clearly see that the specularity is radial on the scene plane and that the isophote is given by the circle of radius $r_-$ (in green), lying inside the outer ring (in orange), whilst the circle of radius $r_+$ lies outside it (in red).
This is verified for all isophotes in figure~\ref{fig:fig1}(b).



\begin{figure}[!h]%
    \centering
    \subfloat[\centering Isocurves for $\kappa=0.55$]{{\includegraphics[width=3.8cm]{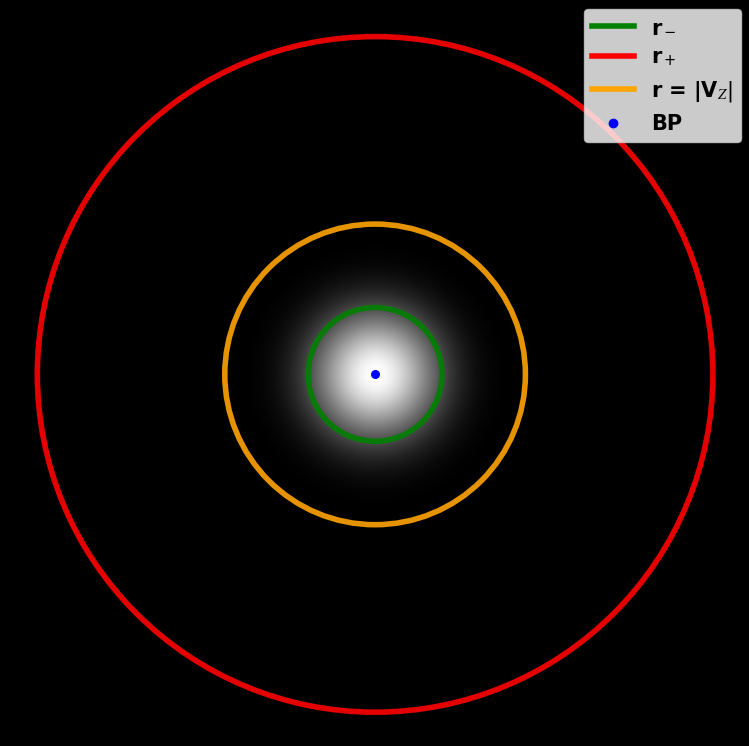} }}%
    \qquad
    \subfloat[\centering Radii against $\kappa$]{{\includegraphics[height=3.7cm,width=3.8cm]{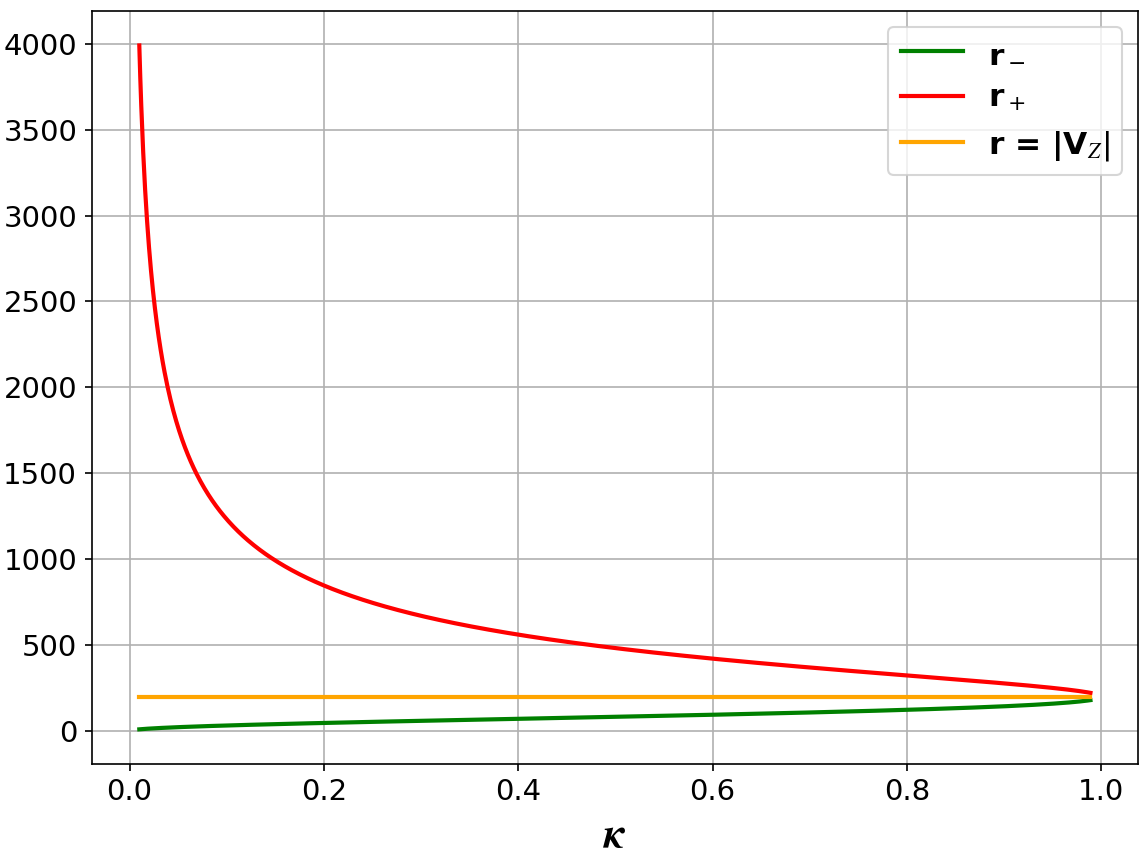} }}%
    \caption{Endoscopic specular isocurves on the scene plane.}%
    \label{fig:fig1}%
\end{figure}

\subsection{Reconstruction Method}
\label{sssec:normal reconstruction} 
Proposition~\ref{proposition2} shows that the specular isophotes are concentric circles on the scene plane.
By property of perspective projection, these circles form ellipses on the image plane in the general case. 
We use this property to reconstruct the plane normal from the specularity observed in the image in closed-form.
We denote as $C\in\mathbb{R}^{3\times3}$ the symmetric matrix representation of the isophote ellipse as a so-called point-conic in the image. 
Background on ellipses and conics is given in~\cite[Chapter 8]{hartley2003multiple}.
We proceed with a given isovalue $0<t<m$, where $m$ is the maximum image intensity (typically $m=255$ for 8-bit images), a given region of interest $\Omega$ containing the specularity and a given intrinsic camera parameter matrix $K\in\mathbb{R}^{3\times3}$ obtained by classical calibration~\cite[Chapter 6]{hartley2003multiple}.
Our method follows four steps.
First, we find the image pixels pertaining to the isophote in $\Omega$ using the Marching Squares algorithm, the 2D version of the popular Marching Cubes algorithm.
Second, we fit the isophote ellipse using the closed-form solution~\cite{fitzgibbon1999direct}.
This gives us matrix $C$.
Third, we transfer the ellipse to the normalised camera plane as $C'=K^\top C K$.
Fourth, we use the circle pose estimation method~\cite[Appendix A]{forsyth1991invariant}, or so-called backprojection problem for one conic correspondence, which provides us with two possible normals $N_\pm$ for the scene plane.
These two normals are due to a concave/convex surface ambiguity, under the assumption that the isophote lies on the tangent plane to the observed surface at the BP.
This ambiguity cannot generally be resolved from a single viewpoint~\cite{zisserman1989information}.
If the circle radius were known, the camera to plane distance could be reconstructed.
However, as opposed to observing physical circles, the circle radius related to a specularity is never known, and the camera to plane distance is thus unrecoverable.


\section{Experimental Results}
\label{sec:results}

\begin{figure*}
\centering
\includegraphics[width=0.9\textwidth]{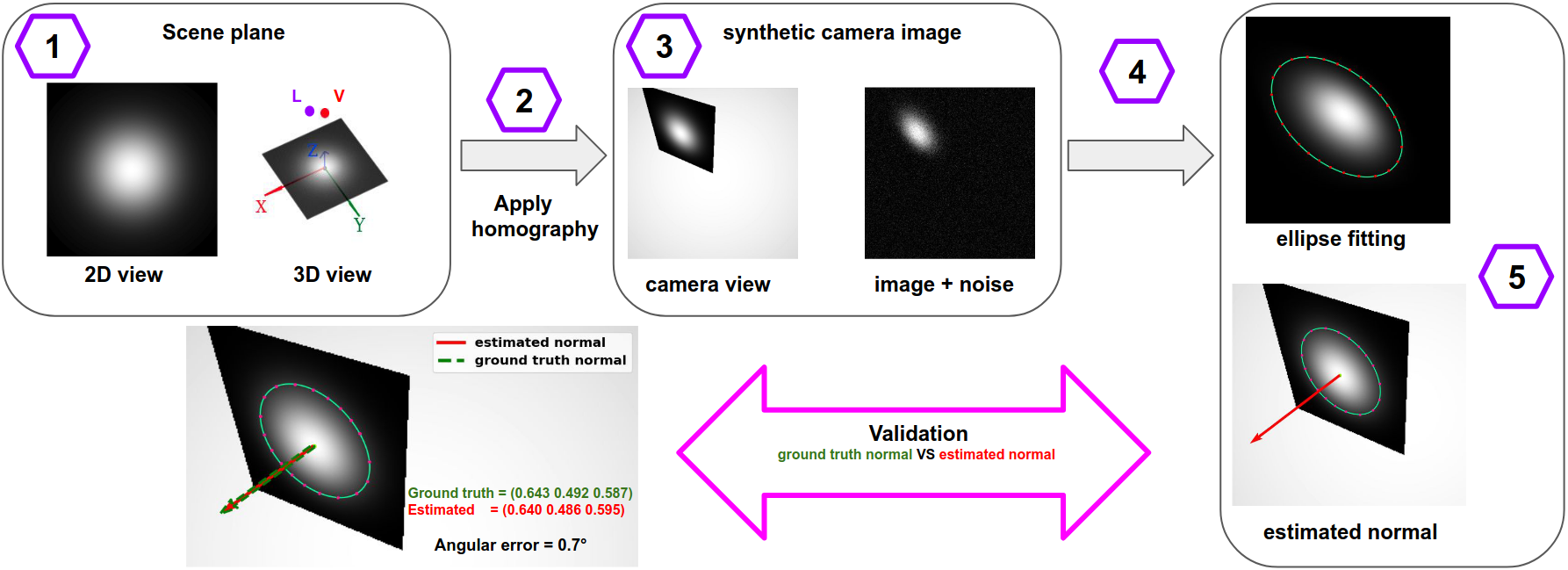}
\caption{\label{fig:fig2} Pipeline for the validation of normal reconstruction from specular isophotes on simulated data.}
\end{figure*} 

\subsection{Synthetic Data}
\label{subsec:s_data}
Figure~\ref{fig:fig2} illustrates the three-stage evaluation pipeline with synthetic data.
In stage 1, we simulate the imaging process with the following three steps. 
1) We compute the scene plane texture map as an $M\times M$ image using equation~\eqref{Eq1} with material roughness $n$. Without loss of generality, we set $V_X= V_Y= 0$ so that the BP lies at the image center. We thus have that $V_Z > 0$ is the camera to scene plane distance. 
We simulate a non-perfect endoscopic setup by choosing $\epsilon > 0$ and setting $L = V + \epsilon (\cos(\alpha),\sin(\alpha),\beta)^\top$, where $\alpha$ and $\beta$ are uniformly distributed in $[0, 2\pi[$ and $[-0.5, 0.5]$, respectively.
2) We apply a projective transformation to the scene plane by means of a homography $H$, to form its sought perspective image $I_s$.
We use a slanted view, as typical of colonoscopy, with a plane normal to camera axis angle $\theta = 58^{\circ}$. We set the camera intrinsics as $f=M$ for the focal length and $c_x=c_y=M/2$ for the principal point.
3) We add white Gaussian noise to the simulated image with variance $\sigma^2$. 
In stage 2, we proceed with the reconstruction process with the following two steps.
4) We use Gaussian smoothing to reduce the effect of noise. We then normalise the image intensities so that the intensity at the BP is 1. 
5) Given the isovalue  $0<t<1$, we run the proposed reconstruction method.
In stage 3, we compute the angular error on each of the estimated normals. 
The default parameters are $\epsilon =0$, $M=406$, $V_Z = 1000$, $n=50$, $\sigma = 5\%$, and $t=0.1$; they lead to the illustration from figure~\ref{fig:fig2}.

The error plots for 1000 independent realisations are shown in figure~\ref{fig:errorplots}.
We vary five important parameters within their range of interest while freezing the others to their default values. First, the noise $\sigma$ is varied from 0\% to 10\% of the intensity range.
We observe that the normal error increases with noise, in average and in standard deviation.
It is on average kept below 1.25°, including the case of very noisy images.
Second, the plane to camera angle $\theta$ is varied from 0 (fronto-parallel view) towards $\frac{\pi}{2}$ (extreme grazing view). We observe a strong growth of the error, with nonetheless a maximum value kept under 7°.
Third, the material roughness $n$ is varied from 30 to 120.
The higher this parameter, the steeper the intensity profile of the specularity, making isophote detection more sensitive to noise.
We indeed observe that the normal error increases linearly with roughness, with a very small slope of approximately $10^{-2}$, without exceeding 1.75°.
Fourth, the isovalue $t$ is varied from 0.02 and 0.8. 
Recall that for a low value the isophote is wide and close to the outer ring whilst for a high value it shrinks closely to the BP.
Therefore the stable configurations should be between these two extremes.
We observe that this indeed happens, with a globally U-shaped normal error curve, reaching a minimum at about 0.5 to 0.6.
Fifth, the imperfect co-location of the light source and camera centre $\epsilon$ is varied between 0 and 800, which, as for $V_Z$ is in  number of pixels. 
Specifically, we first centre the scene plane's window at the BP and choose $n=100$ to explore larger space of model imperfections while ensuring that the model errors are not mixed with errors related to partial view. Recall that for a large value of $\epsilon$ the isophote is no longer a circle in the scene plane.
We observe that the normal error remains stable for a light source and camera located inside a cylinder of radius 200~mm and height 200~mm (\textit{i.e.}, for $\epsilon \leq 200$).
This is much larger than the distance which may exist in the design of most endoscopes.

\begin{figure*}[!h]
    \includegraphics[width=.195\textwidth]{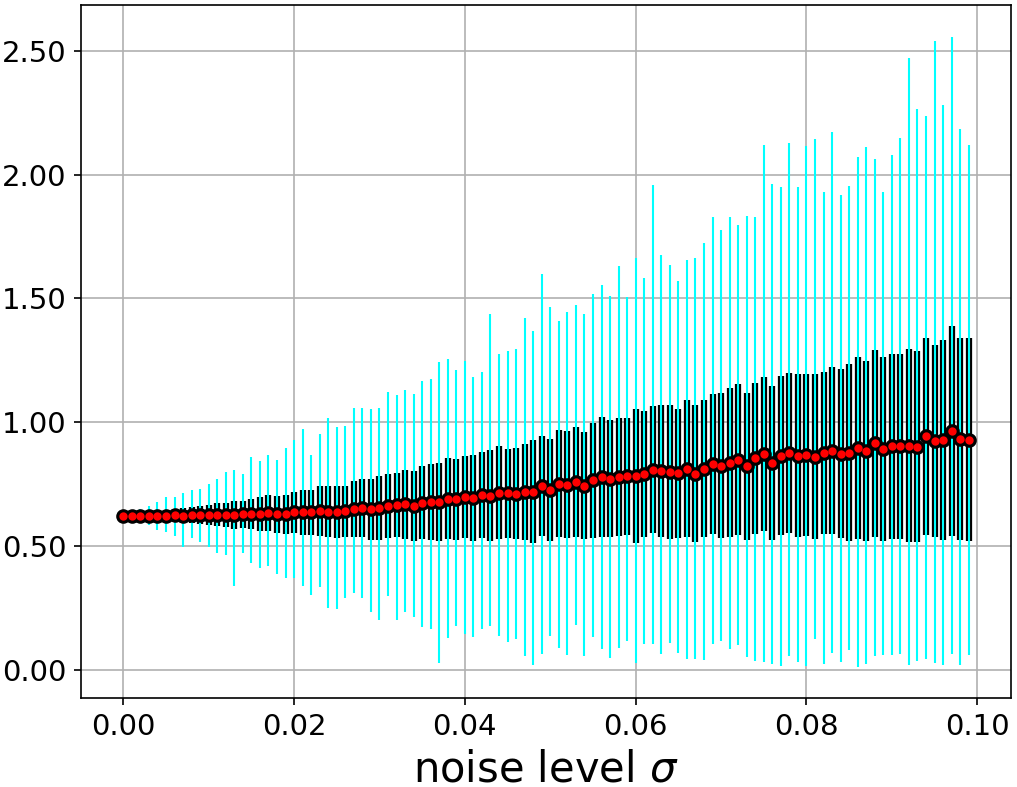}\hfill
    \includegraphics[width=.195\textwidth]{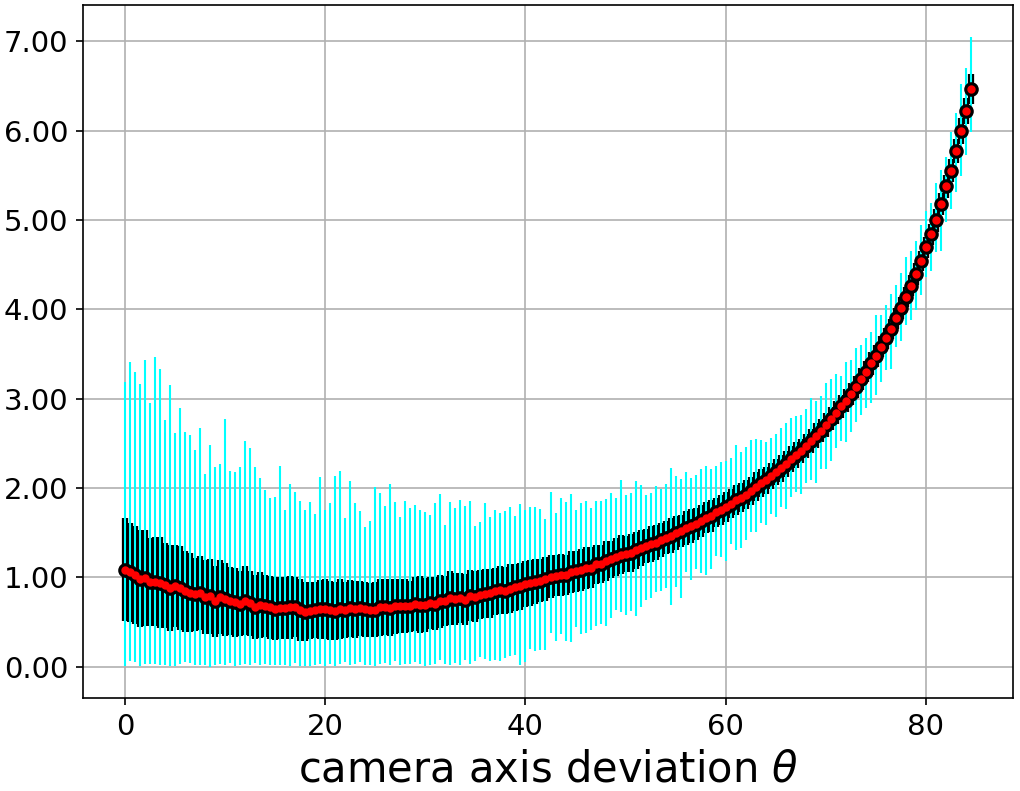}\hfill
    \includegraphics[width=.195\textwidth]{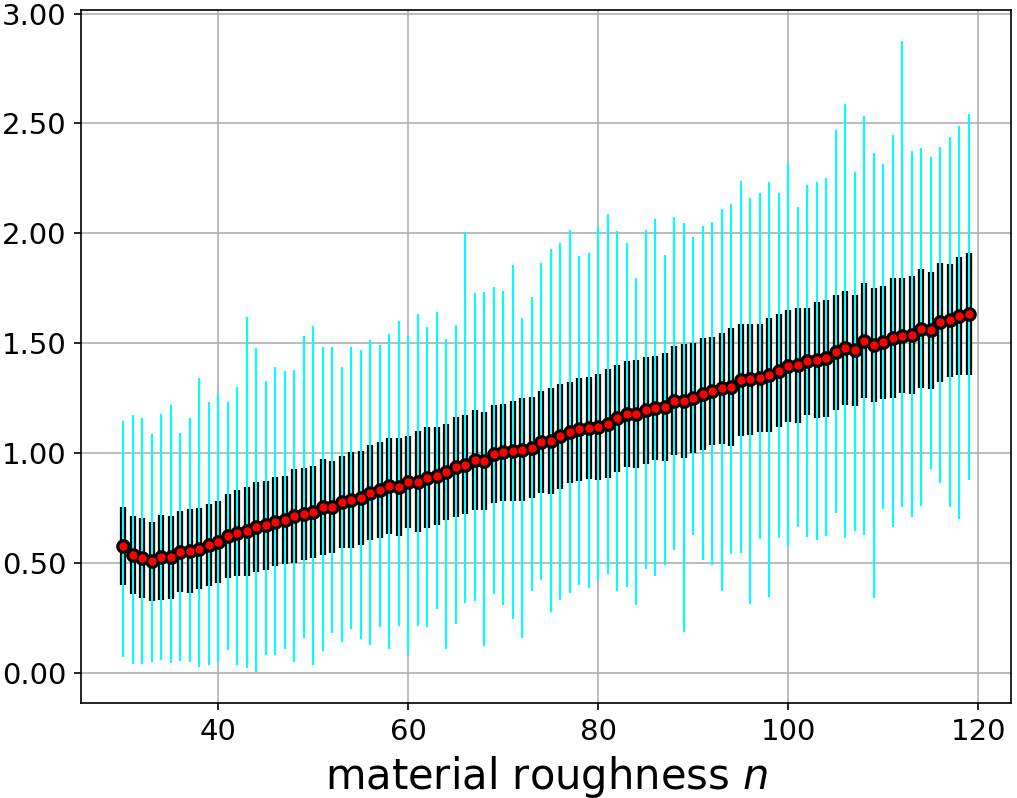}\hfill
    \includegraphics[width=.195\textwidth]{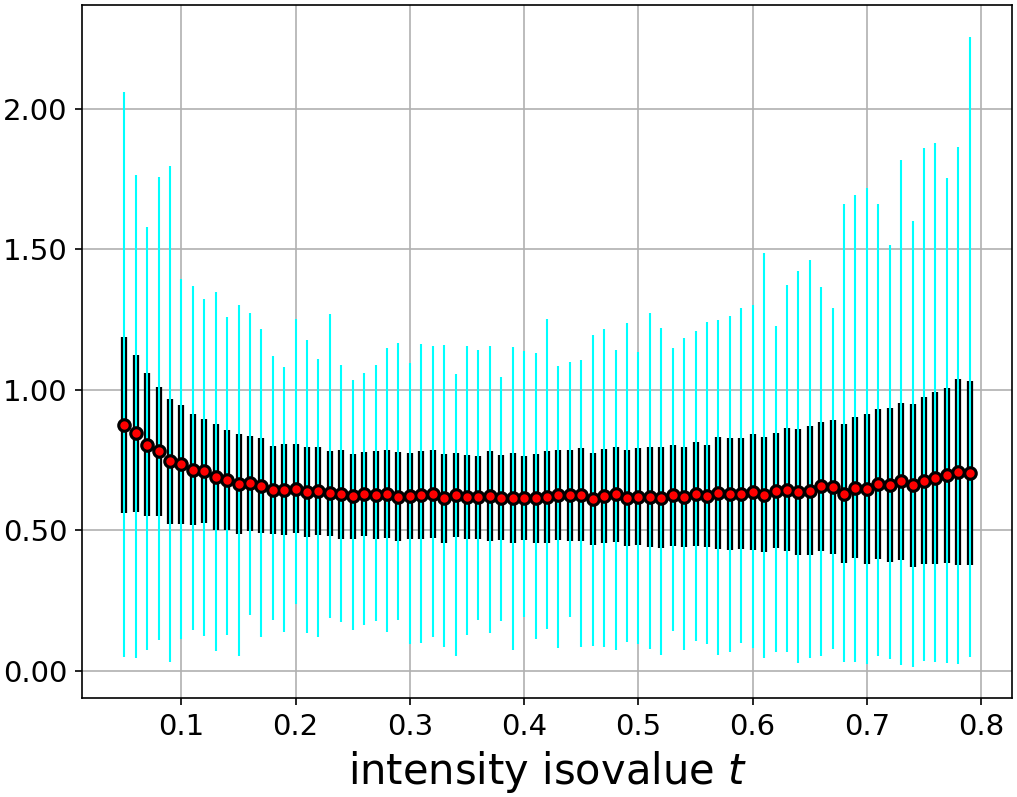}
    \includegraphics[width=.195\textwidth]{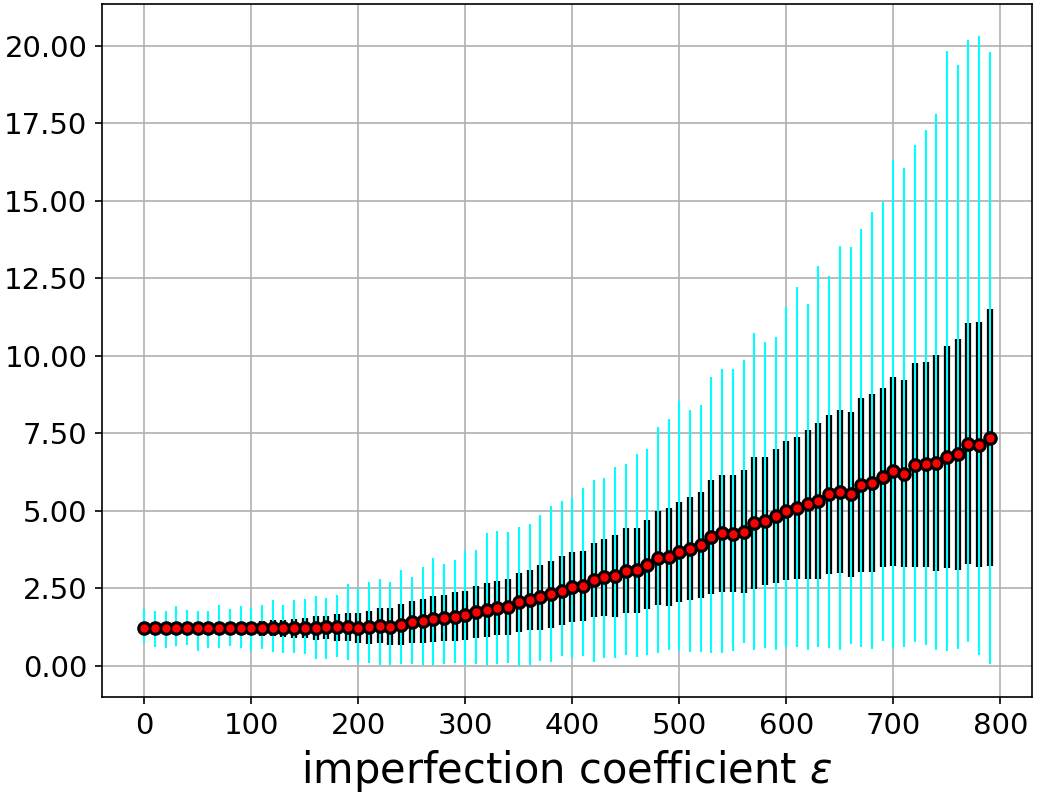}

    \caption{Evaluation results on synthetic data, giving the average normal error $\bar E$ in degrees. The black bars represent the error standard deviation and the cyan bars represent the interval $ [\bar{E} -\min(E), \max (E)-\bar{E}]$.}\label{fig:errorplots}
\end{figure*}
\subsection{Real Data}
\label{subsec:clinical_data}
We used our method in two concrete cases, laparoscopic liver resection and colonoscopy.
In both cases, the practical problem is to obtain an interventional 3D model from images.
This is difficult because the observed structures are very deformable, precluding the use of many vision techniques such as Structure-from-Motion.
It is thus tremendously important to be able to exploit single-image 3D reconstruction, for which our method is well-adapted.
First, the camera can be easily calibrated when the procedure starts.
Second, because the observed structures are moist, the images contain many specularities.
Third, the specularities are typically small in size, and thus supported by a locally planar surface region.

\subsubsection{Laparoscopic Liver Resection}
\label{subsubsec:laparoscopic}
We ran our method on five images with a total of 26 specularities extracted from five procedures collected in our hospital under ethical approval IRB00008526-2019-CE58 issued by CPP Sud-Est VI in Clermont-Ferrand, France. A representative example is shown in figure~\ref{fig:liver}.
We disambiguated the normal-pairs manually and performed a comparison with~\cite{koo2017deformable}, a method which registers a preoperative 3D model to the image, obtaining an angular difference of $7.11^\circ\pm4.86^\circ$.
This shows a very good quantitative agreement between our method and~\cite{koo2017deformable}.

\begin{figure}[!h]%
    \centering
    \subfloat[\centering Laparoscopic image]{{\includegraphics[height=3cm,width=4.6cm]{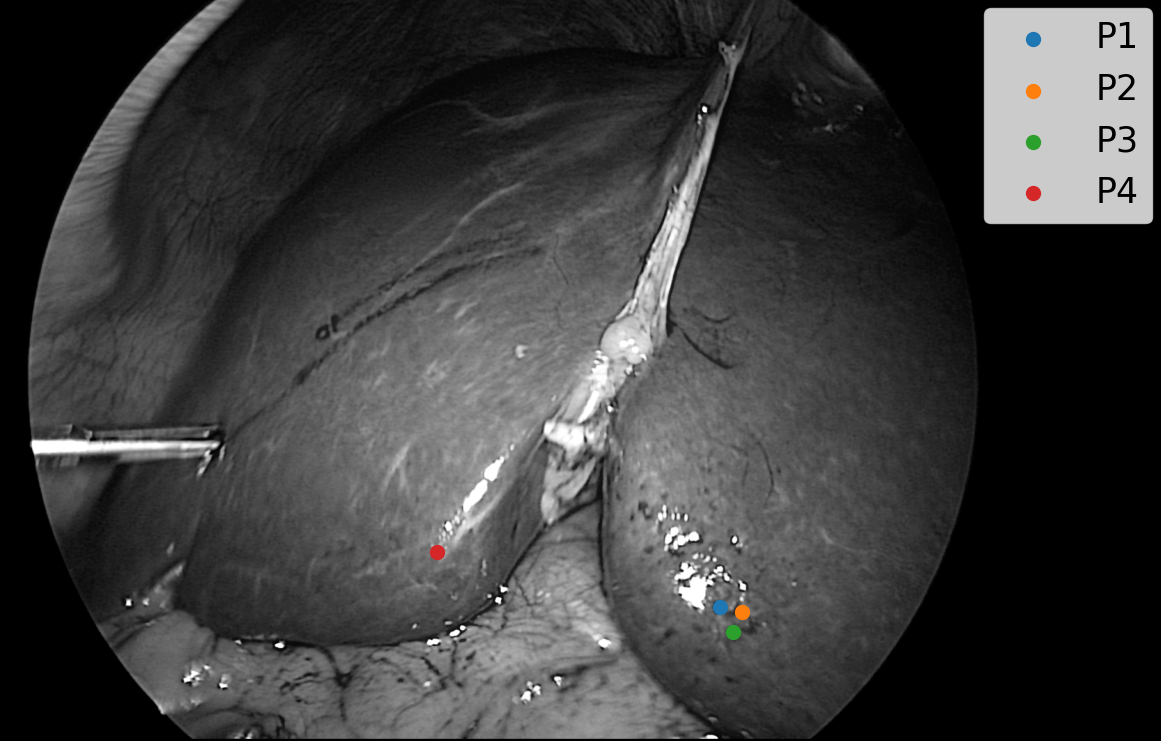} }}%
    \qquad
    \subfloat[\centering Results for P1]{{\includegraphics[width=2.8cm]{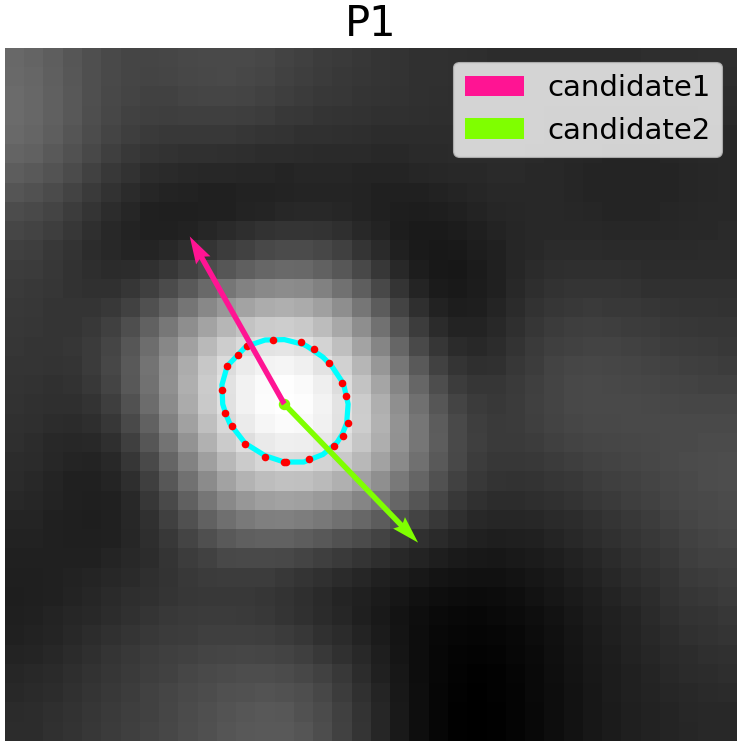} }}%
       \qquad
    \subfloat[\centering Reconstructed VS mesh normals]{{\includegraphics[height=2.5cm,width=4.2cm]{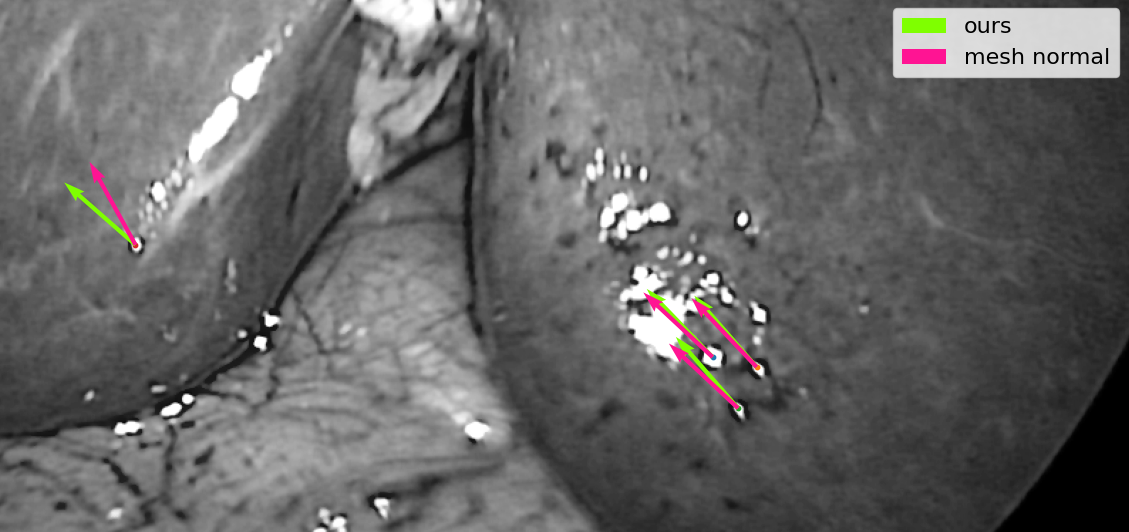} }}%
    \qquad
    \subfloat[\centering Reference 3D model]{{\includegraphics[width=3.5cm]{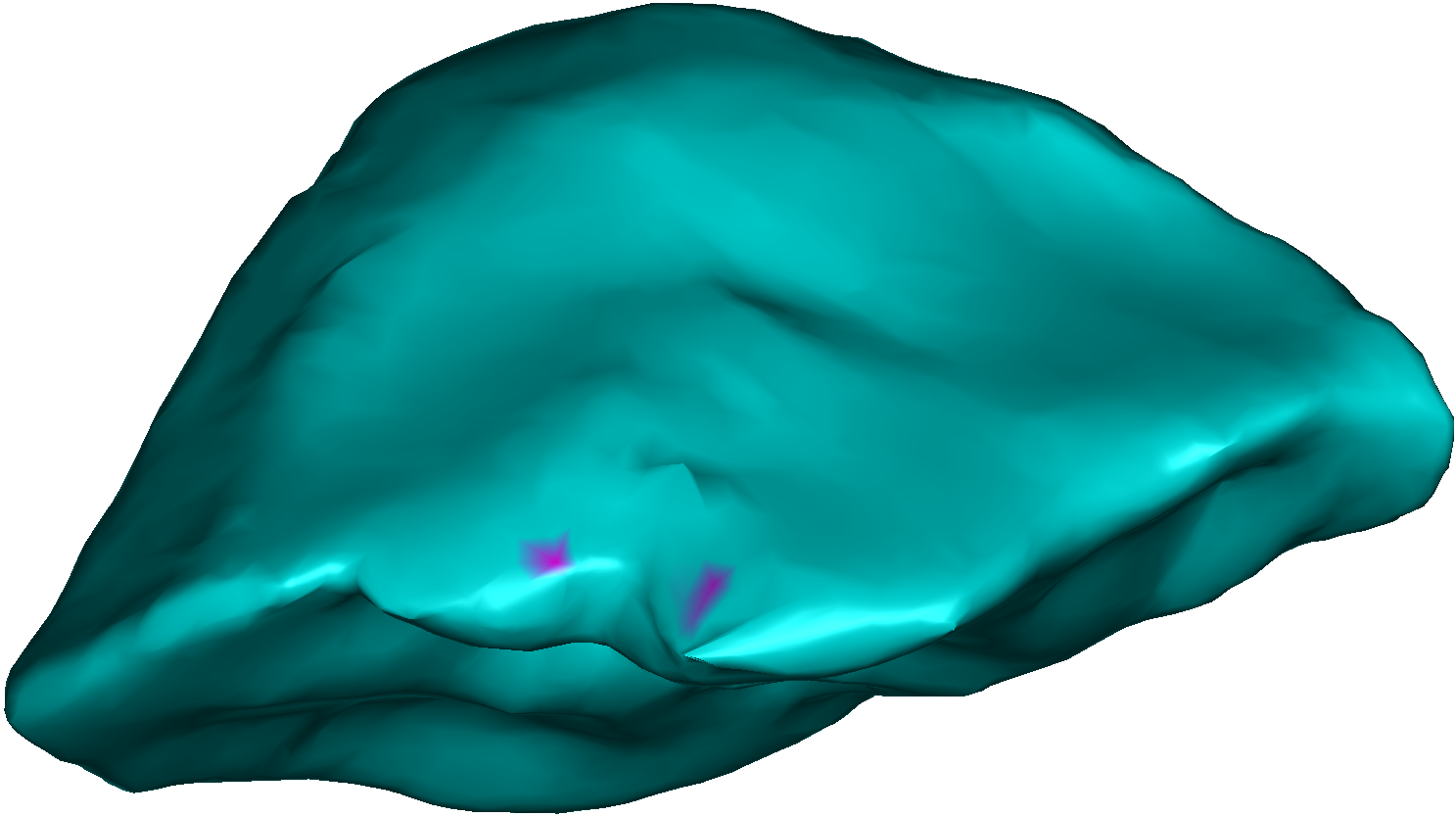} }}%
     \caption{Example of normal reconstruction for a 2D laparoscopic image. (a) The four specularities used. (b) Close-up and reconstructed normals for P1. (c) Reconstructed normals for our method compared to  preoperative model registration~\cite{koo2017deformable}. (d) Location of the four specularities on the registered preoperative model.}%
    \label{fig:liver}%
\end{figure}

\subsubsection{Colonoscopy}
\label{subsubsec:colonoscopic}
We ran our method on $1248\times 1080$ images extracted from colonoscopy procedures within the IRB approved protocol of the Endomapper project.
We show representative examples in figure~\ref{fig:colon} with 11 reconstructed normals.
We did not find a method which would provide a satisfying 3D reconstruction for comparison.
Nonetheless, the reconstructed normals we obtained are qualitatively sound, showing the potential applicability of our method to colonoscopy.
We considered using the simulated data from EndoSLAM~\cite{ozyoruk2021endoslam}, for which however the surface is made of fronto-parallel planes whose normal is thus always $(0,0,1)^\top$. 
In contrast,~\cite{rau2019implicit} provides smoother depth maps but the surface triangulation is not dense enough to create realistic local specularities, as shown in figure~\ref{fig:endoslamVSucl}.





\begin{figure}[!h]
\centering
\includegraphics[width=0.5\textwidth]{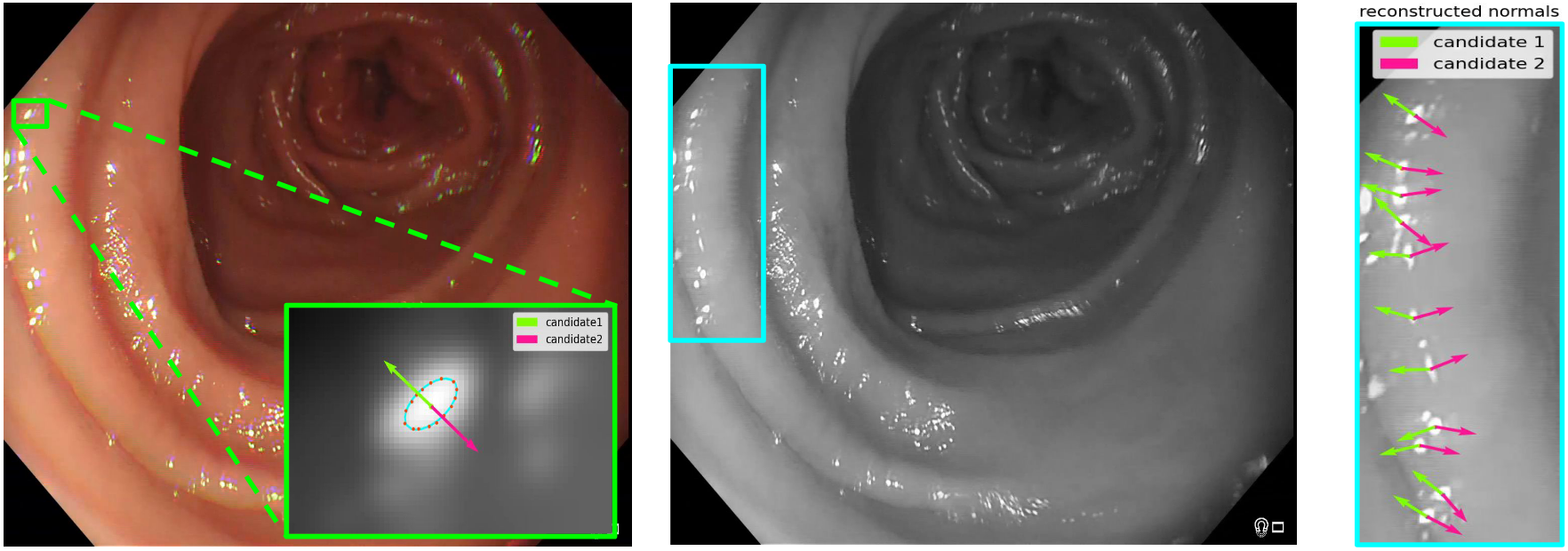}
\caption{\label{fig:colon} Sparse normal reconstruction for the colon surface. We show the two possible orientations of each local tangent plane containing an isolated specular highlight.} 
\end{figure} 

\begin{figure}[!h]
\centering
\includegraphics[width=0.5\textwidth]{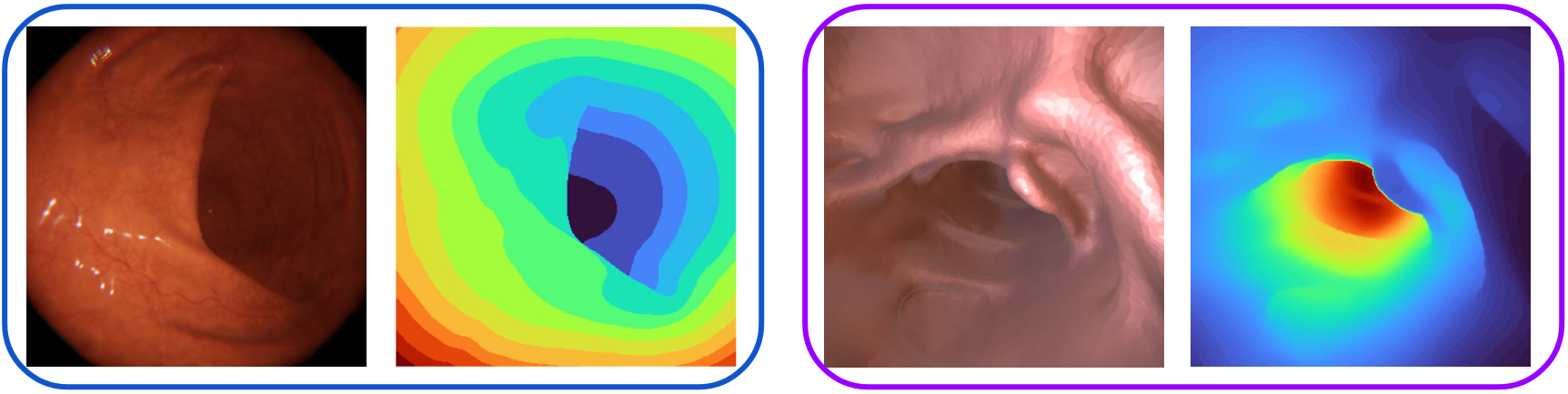}
\caption{\label{fig:endoslamVSucl} Excerpts of publicly available datasets, left: EndoSLAM~\cite{ozyoruk2021endoslam} and right:~\cite{rau2019implicit}.}
\end{figure} 


\section{Conclusion}
\label{sec:conclusion}
We have introduced a geometric model of the specular isophotes in endoscopy, from which we have derived a normal reconstruction method. 
This is based on the general specular model~\cite{bartoli2019highlight}, which  we have specialised to endoscopy and extended to 3D reconstruction. 
We have obtained convincing results in two clinical applications.
Our future work will focus on two main points.
First, we will design automatic specularity detection by means of combining deep learning and the elliptic shape prior we have derived.
This will increase the applicability of our method and cope with the mixed specularities of non-elliptical shape owing the inter-reflections.
Second, we will incorporate our method into the tubular-NRSfM method applied to the colon~\cite{sengupta2021colonoscopic}.

\textbf{Declarations}
Informed consent was obtained from all individual participants included in the study. This article
does not contain any studies with animals performed by any of the authors.


\bibliographystyle{IEEEbib}
\bibliography{strings,refs}

\end{document}